\documentclass{article}
\usepackage{ijcai22}

\usepackage{times}
\usepackage{soul}
\usepackage{url}
\usepackage[hidelinks]{hyperref}
\usepackage[utf8]{inputenc}
\usepackage[small]{caption}
\usepackage{graphicx}
\usepackage{amsmath, amsfonts}
\usepackage{amsthm}
\usepackage{booktabs}
\usepackage{algorithm}
\usepackage{algorithmic}
\usepackage{multirow}
\usepackage{color}
\usepackage{comment}
\usepackage{soul}

\newtheorem{theorem}{Theorem}
\newcommand\numberthis{\addtocounter{equation}{1}\tag{\theequation}}

\title{DeepExtrema: A Deep Learning Approach for Forecasting Block Maxima\\ in Time Series Data}

\author{
Asadullah Hill Galib$^1$\and
Andrew McDonald$^1$\and
Tyler Wilson$^1$\and
Lifeng Luo$^2$\And
Pang-Ning Tan$^1$\
\affiliations
$^1$Department of Computer Science \& Engineering, Michigan State University\\
$^2$Department of Geography, Michigan State University
\emails
\{galibasa, mcdon499, wils1270, lluo, ptan\}@msu.edu
}

\begin{document}

\maketitle

\begin{abstract}
Accurate forecasting of extreme values in time series is critical due to the significant impact of extreme events on human and natural systems. 
This paper presents DeepExtrema, a novel framework that combines a deep neural network (DNN) with generalized extreme value (GEV) distribution to forecast the block maximum value of a time series. Implementing such a network is a challenge as the framework must preserve the inter-dependent constraints among the GEV model parameters even when the DNN is initialized. 
We describe our approach to address this challenge and present an architecture that enables both conditional mean and quantile prediction of the block maxima. The extensive experiments performed on both real-world and synthetic data demonstrated the superiority of DeepExtrema compared to other baseline methods.
  
\end{abstract}

\section{Introduction}

Extreme events such as droughts, floods, and severe storms occur when the values of the corresponding geophysical variables (such as temperature, precipitation, or wind speed) reach their highest or lowest point during a period or surpass a threshold value. Extreme events have far-reaching consequences for both humans and the environment. For example, four of the most expensive hurricane disasters in the United States since 2005---Katrina, Sandy, Harvey, and Irma---have each incurred over \$50 billion in damages 
with enormous death tolls~\cite{usgao}. Accurate forecasting of the extreme events \cite{wang2021johan} is therefore crucial as it not only helps provide timely warnings to the public but also enables emergency managers and responders to better assess the risk of potential hazards caused by future extreme events. 


\begin{figure}
	\centering
	\includegraphics[width=2.5in]{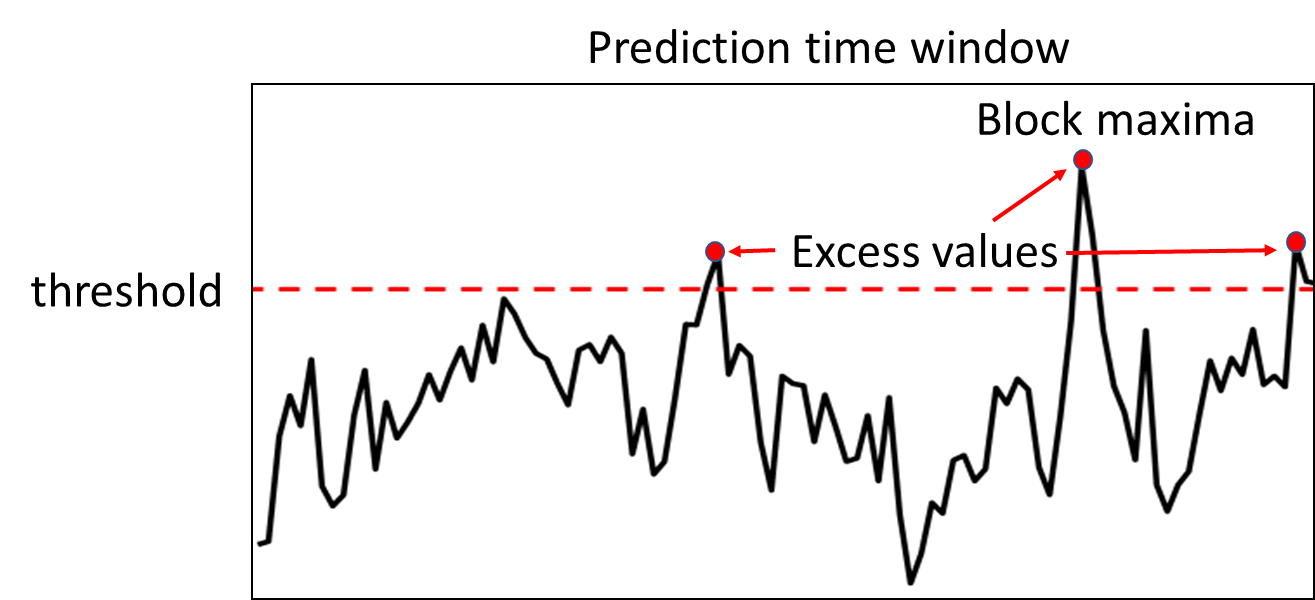}
	\caption{Types of extreme values in a given time window.}
	\label{fig:evt}
	\vspace{-0.1cm}
\end{figure}

Despite its importance, forecasting time series with extremes can be tricky as the extreme values may have been ignored as outliers during training to improve the generalization performance of the model. Furthermore, as current approaches are mostly designed to minimize the mean-square prediction error, their fitted models focus on predicting the conditional expectation of the target variable rather than its extreme values \cite{bishop2006pattern}.  
Extreme value theory (EVT) offers a statistically well-grounded approach to derive the limiting distribution governing a sequence of extreme values~\cite{coles2001introduction}. The two most popular distributions studied in EVT are the generalized extreme value (GEV) and generalized Pareto (GP) distributions. Given a prediction time window, GEV governs the distribution of its block maxima, whereas GP is concerned with the distribution of excess values over a certain threshold, as shown Figure \ref{fig:evt}. This paper focuses on forecasting the block maxima as it allows us to assess the worst-case scenario in the given forecast time window and avoids making ad-hoc decisions related to the choice of excess threshold to use for the GP distribution. Unfortunately, classical EVT has limited capacity in terms of modeling highly complex, nonlinear relationships present in time series data. 
For example, \cite{kharin2005estimating} uses a simple linear model with predictors to infer the parameters of a GEV distribution.

Deep learning methods have grown in popularity in recent years due to their ability to capture nonlinear  dependencies in the data. Previous studies have utilized a variety of deep neural network architectures for time series modeling, including long short-term memory networks \cite{sagheer2019time,masum2018multi,laptev2017time},  convolutional neural networks \cite{bai-empirical-2018,zhao2017convolutional,yang2015deep}, encoder-decoder based RNN \cite{peng2018multi}, and attention-based models \cite{zhang2019lstm,aliabadi2020attention}. 
However, these works are mostly focused on predicting the conditional mean of the target variable. 
While there have some recent attempts to incorporate EVT into deep learning~\cite{wilson2022,ding2019modeling,polson2020deep}, they are primarily focused on modeling the tail distribution, i.e., excess values over a threshold, using the GP distribution, rather than forecasting the block maxima using the GEV distribution. Furthermore, instead of inferring the distribution parameters from data, some methods \cite{ding2019modeling} assume that the parameters are fixed at all times and can be provided as user-specified hyperparameters while others \cite{polson2020deep} do not enforce the necessary constraints on parameters of the extreme value distribution.


Incorporating the GEV distribution into the deep learning formulation presents many technical challenges. First, the GEV parameters must satisfy certain positivity constraints to ensure that the predicted distribution has a finite bound~\cite{coles2001introduction}. Maintaining these constraints throughout the training process is a challenge since the model parameters depend on the observed predictor values in a mini-batch. Another challenge is the scarcity of data since there is only one block maximum value per time window. This makes it hard to accurately infer the GEV parameters for each window from a set of predictors. Finally, the training process is highly sensitive to model initialization. 
For example, the random initialization of a deep neural network (DNN) can easily violate certain regularity conditions of the GEV parameters estimated using maximum likelihood (ML) estimation. An improper initialization may lead to $\xi$ estimates that defy the regularity conditions. For example, if $\xi < -0.5$, then  resulting ML estimators may not have the standard asymptotic properties \cite{smith1985}, whereas if $\xi > 1$, then the conditional mean is not well-defined. Without proper initialization,
the DNN will struggle to converge to a feasible solution with acceptable values of the GEV parameters. Thus, controlling the initial estimate of the parameters is difficult but necessary. 

To overcome these challenges, we propose a novel framework called \texttt{DeepExtrema} that utilizes the GEV distribution to characterize the distribution of block maximum values for a given forecast time window. The parameters of the GEV distribution are estimated using a DNN, which is trained to capture the nonlinear dependencies in the time series data. This is a major departure from previous work by \cite{ding2019modeling}, where the distribution parameters are assumed to be a fixed, user-specified hyperparameter. 
\texttt{DeepExtrema} reparameterizes the GEV formulation to ensure that the DNN output is compliant with the GEV positivity constraints. In addition, \texttt{DeepExtrema} offers a novel, model bias offset mechanism in order to ensure that the 
regularity conditions of the GEV parameters are satisfied from the beginning.


In summary, the main contributions of the paper are: 
\begin{enumerate}
    \item We present a novel framework to predict the block maxima of a given time window by incorporating GEV distribution into the training of a DNN. 
    \item We propose a reformulation of the GEV constraints to ensure they can be enforced using activation functions in the DNN. 
    \item We introduce a model bias offset mechanism to ensure that the DNN output preserves 
    the regularity conditions of the GEV parameters despite its random initialization. 
    \item We perform extensive experiments on both real-world and synthetic data to demonstrate the effectiveness of \texttt{DeepExtrema} compared to other baseline methods.
\end{enumerate}

\section{Preliminaries}

\subsection{Problem Statement}
Let $z_1,z_2,\cdots,z_T$ be a time series of length $T$. Assume the time series is partitioned into a set of time windows, where each window $[t-\alpha,t+\beta]$ contains a sequence of predictors, $x_t = (z_{t-\alpha},z_{t-\alpha+1},\cdots,z_{t})$, and target, $\Tilde{y}_t = (z_{t+1},z_{t+2},\cdots,z_{t+\beta})$. Note that $\beta$ is known as the forecast horizon of the prediction. For each time window, let $y_t = \max_{\tau \in \{1,\cdots,\beta\}} z_{t+\tau}$ be the block maxima of the target variable at time $t$. Our time series forecasting task is to estimate the block maxima, $\hat{y}_t$,
as well as its upper and lower quantile estimates, $\hat{y}_U \text{ and } \hat{y}_L$, of a future time window based on current and past data, $x_t$.



\subsection{Generalized Extreme Value Distribution}
\label{sec:gev}

The GEV distribution governs the distribution of block maxima in a given window. 
Let $Y = \max \{z_1, z_2, \cdots, z_t\}$. If there exist sequences of constants ${a_t > 0}$ and ${b_t}$ such that \[ Pr{(Y - b_t)/a_t \leq y} \rightarrow{G(y)} \quad \text{as } t  \rightarrow{\infty} \] for a non-degenerate distribution $G$, then the cumulative distribution function $G$ belongs to a family of GEV distribution of the form \cite{coles2001introduction}:
\begin{equation} 
    G(y) = \exp\bigg\{ -\bigg[1+\xi(\frac{y-\mu}{\sigma})\bigg]^{-1/\xi}\bigg\}
    \label{gev}
\end{equation}

The GEV distribution is characterized by the following parameters: $\mu$ (location), $\sigma$ (scale), and $\xi$ (shape). The expected value of the distribution is given by
\begin{equation} 
y_{mean} = \mu + \frac{\sigma}{\xi}\bigg[\Gamma(1-\xi) -1 \bigg]
\label{meanval}
\end{equation}
where  $\Gamma(x)$ denotes the gamma function of a variable 
$x>0$. Thus, $y_{mean}$ is only well-defined for $\xi<1$.  Furthermore, 
the $p^\textit{th}$ quantile of the GEV distribution, $y_p$, can be calculated as follows:    
\begin{equation} 
    y_p = \mu + \frac{\sigma}{\xi}\bigg[(- \log p)^{-\xi} -1\bigg]
    \label{quantile}
\end{equation}

Given $n$ independent block maxima values, $\{y_1, y_2, \cdots,y_n\}$, with the distribution function given by Equation \eqref{gev} and assuming $\xi \neq 0$, its log-likelihood function is given by:
\begin{align*}
\ell_{GEV}(\mu, \sigma, \xi) & = -n \log\sigma - (\frac{1}{\xi} + 1)\,\sum_{i=1}^n \: \log (1+\xi\,\frac{y_i-\mu}{\sigma}) 
\\&  \quad- \sum_{i=1}^n\:(1+\xi\,\frac{y_i-\mu}{\sigma})^{-1/\xi} \quad  \numberthis \label{loglikelihood}
\end{align*}

The GEV parameters $(\mu, \sigma, \xi)$ can be estimated  using the maximum likelihood (ML) approach 
by maximizing \eqref{loglikelihood} subject to the following positivity constraints: 
\begin{equation} 
    \sigma > 0 \ \ \ \ \textrm{and} \ \ \ \ 
   \forall\,{i}: 1+ \frac{\xi}{\sigma}(y_i-\mu) > 0 
   \label{c-1}
\end{equation}

In addition to the above positivity constraints, the shape parameter $\xi$ must be within certain range of values in order for the ML estimators to exist and have regular asymptotic properties \cite{coles2001introduction}. Specifically, the ML estimators have regular asymptotic properties as long as $\xi > -0.5$. Otherwise, if 
$-1< \xi < -0.5$, then the ML estimators may exist but will not have regular asymptotic properties. Finally, the ML estimators do not exist if $\xi < - 1$ \cite{smith1985}.

\section{Proposed Framework: DeepExtrema}
This section presents the proposed \texttt{DeepExtrema} framework for predicting the block maxima of a given time window. The predicted block maxima $\hat{y}$ follows a GEV distribution, whose parameters are conditioned on observations of the predictors $x$. 
Figure \ref{overall} provides an overview of the \texttt{DeepExtrema} architecture. Given the input  predictors $x$, the framework uses a stacked LSTM network to learn a representation of the time series. The LSTM will output a latent representation, which will used by a fully connected layer to generate the GEV parameters: 
\begin{equation} 
(\mu, \sigma, \xi_u, \xi_l) = LSTM(x)
\label{LSTM}
\end{equation}
where $\mu$, $\sigma$, and $\xi$'s are the location, shape, and scale parameters of the GEV distribution. Note that $\xi_u$ and $\xi_l$ are the estimated parameters due to reformulation of the GEV constraints, which will be described in the next subsection.

The proposed Model Bias Offset (MBO) component performs bias correction on the estimated GEV parameters to ensure that the LSTM outputs preserve  
the regularity conditions of the GEV parameters and generate a feasible value for $\xi$ irrespective of how the network was initialized. The GEV parameters are subsequently provided to a fully connected layer to obtain point estimates of the block maxima, which include its expected value $\hat{y}$ as well as upper and lower quantiles, $\hat{y}_U$ and $\hat{y}_L$, using the equations given in \eqref{meanval} and \eqref{quantile}, respectively. The GEV parameters are then used to compute the negative log-likelihood of the estimated GEV distribution, 
which will be combined with the root-mean-square error (RMSE) of the predicted block maxima to determine the overall loss function. 
Details of the different components are described in the subsections below.

\begin{figure}[t]
\centering
\includegraphics[width=0.91\columnwidth]{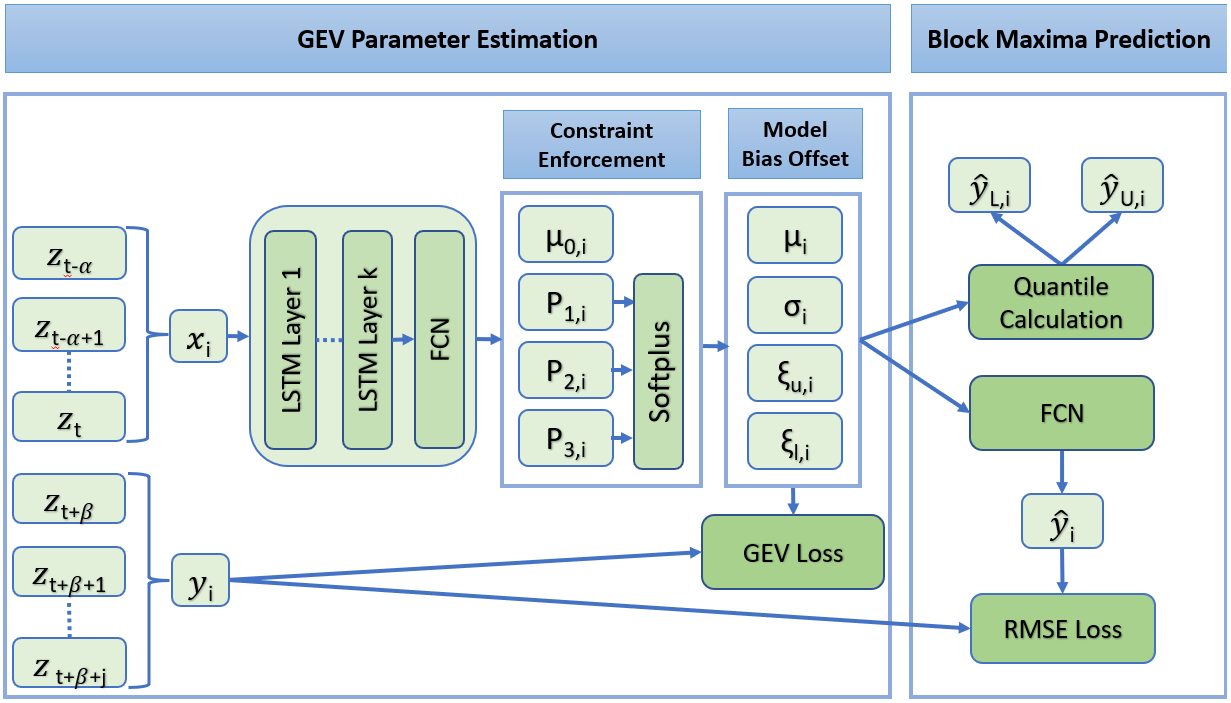} 
\caption{Proposed \texttt{DeepExtrema} framework for predicting block maxima using GEV distribution.}
\label{overall}
\end{figure}

\subsection{GEV Parameter Estimation}
\label{sec:gev_est}

Let $\mathcal{D} = \{(x_i,y_i)\}_{i=1}^n$ be a set of training examples, where each $x_i$ denotes the predictor time series and $y_i$ is the corresponding block maxima for time window $i$. A na\"ive approach is to assume that the GEV parameters $(\mu,\sigma,\xi)$ are constants for all time windows. This can be done by fitting a global GEV distribution to the set of block maxima values $y_i$'s using the maximum likelihood approach given in \eqref{loglikelihood}. Instead of using a global GEV distribution with fixed parameters, our goal is to learn the parameters  $(\mu_i,\sigma_i,\xi_i)$ of each window $i$ using the predictors $x_i$. The added flexibility enables the model to improve the accuracy of its block maxima prediction, especially for non-stationary time series. 


The estimated GEV parameters generated by the LSTM must satisfy the two positivity constraints given by the inequalities in 
\eqref{c-1}. 
While the first positivity constraint on $\sigma_i$ 
is straightforward to enforce, maintaining the second one 
is harder as it involves a nonlinear relationship between $y_i$ and the estimated GEV parameters, $\xi_i$, $\mu_i$, and $\sigma_i$. The GEV parameters may vary from one input $x_i$ to another, and thus, learning them from the limited training examples is a challenge. Worse still, some of the estimated GEV parameters could be erroneous, especially at the initial rounds of the training epochs, making it difficult to satisfy the constraints throughout the learning process.  

To address these challenges, we propose a reformulation of the second constraint in \eqref{c-1}. This allows the training process to proceed even though the second constraint in \eqref{c-1} has yet to be satisfied especially in the early rounds of the training epochs. 
Specifically, we relax the hard constraint by adding a small tolerance factor, $\tau > 0$, as follows:
\begin{equation}
\forall\,{i}: 1+ \frac{\xi}{\sigma}(y_i-\mu) + \tau \ge 0.
\label{c-2}
\end{equation}
The preceding soft constraint allows for minor violations of the second constraint in \eqref{c-1} as long as $1+ \frac{\xi}{\sigma}(y_i-\mu) > -\tau$ for all time windows $i$. Furthermore, to ensure that the inequality holds for all $y_i$'s, we reformulate the constraint in \eqref{c-2} in terms of $y_{\min} = \min_i y_i$ and $y_{\max} = \max_i y_i$ as follows:
\begin{theorem}
Assuming $\xi \neq 0$, the soft constraint in \eqref{c-2} can be reformulated into the following bounds on $\xi$: 
\begin{equation} 
-\frac{\sigma}{y_{\mathrm{max}} -\mu} \; (1+\tau) \le \xi \le  \frac{\sigma}{\mu - y_{\mathrm{min}}} \; (1+\tau) \label{c2new} 
\end{equation}
where $\tau$ is the tolerance on the constraint in \ref{c-1}.
\end{theorem}

\begin{proof}
For the lower bound on $\xi$, set $y_i$ to be $y_{\max}$ in \eqref{c-2}:
\begin{align*}
1+ \frac{\xi}{\sigma}(y_{\mathrm{max}}-\mu) + \tau \ge 0 
\implies &\frac{\xi}{\sigma}(y_{\mathrm{max}}-\mu) \ge -\:(1 + \tau ) \\
\implies&\xi \ge -\frac{\sigma}{(y_{\mathrm{max}}-\mu)}\:(1+\tau)
\end{align*}
To obtain the upper bound on $\xi$, set $y_i$ to be $y_{\min}$ in \eqref{c-2}: 
\begin{align*}
1+ \frac{\xi}{\sigma}(y_{\mathrm{min}}-\mu) + \tau \ge 0 
\implies& \frac{\xi}{\sigma}(\mu-y_{\mathrm{min}}) \le \:(1 + \tau ) \\
\implies&\xi \le \frac{\sigma}{(\mu-y_{\mathrm{min}})}\:(1+\tau)
\end{align*}
\end{proof}

\noindent Following Theorem 1, the upper and lower bound constraints on $\xi$ in \eqref{c2new} can be restated as follows: 
\begin{align} \label{c2new_split}
    & \frac{\sigma}{\mu - y_{\mathrm{min}}} \; (1+\tau) - \xi \ge 0 \nonumber \\ 
   & \xi + \frac{\sigma}{y_{\mathrm{max}} -\mu} \; (1+\tau) \ge 0  
\end{align}
The reformulation imposes lower and upper bounds on $\xi$, which can be used to re-parameterize the second constraint in \eqref{c-1}. 
Next, we describe how the reformulated  constraints in \eqref{c2new_split} can be enforced by \texttt{DeepExtrema} in a DNN. 

Given an input $x_i$, \texttt{DeepExtrema} will generate the following four outputs: $\mu_i$, $P_{1i}$, $P_{2i}$, and $P_{3i}$. A softplus activation function, $softplus(x) = \log{(1+\exp{(x)})}$, which is a smooth approximation to the ReLU function, is used to enforce the non-negativity constraints associated with the GEV parameters. 
The scale parameter $\sigma_i$ can be computed using the softplus activation function on $P_{1i}$ as follows:
\begin{equation} \label{r-1}
    \sigma_i  = softplus{(P_{1i})}  \\
\end{equation}
This ensures the constraint $\sigma_i \ge 0$ is met. 
The lower and upper bound constraints on $\xi_i$ given by the inequalities in \eqref{c2new_split} are enforced using the softplus function on 
$P_{2i}$ and $P_{3i}$:  
\begin{align} 
    &\frac{\sigma_i}{\mu_i - y_{\mathrm{min}}} \; (1+\tau) - \xi_{u,i} = softplus{(P_{2i})} \nonumber \\
    &\frac{\sigma_i}{y_{\mathrm{max}} -\mu_i} \; (1+\tau) + \xi_{l,i} = softplus{(P_{3i})}
\label{r-2}
\end{align}
By re-arranging the above equation, we obtain
\begin{align} 
    &\xi_{u,i} = \frac{\sigma_i}{\mu_i - y_{\mathrm{min}}} \; (1+\tau) - softplus{(P_{2i})} \nonumber \\
    &\xi_{l,i} = softplus{(P_{3i})} - \frac{\sigma_i}{y_{\mathrm{max}} -\mu_i} \; (1+\tau)
    \label{xi}
\end{align}
\texttt{DeepExtrema} computes the upper and lower bounds on $\xi_i$ using the formulas in \eqref{xi}. During training, it will minimize the distance between $\xi_{u,i}$ and $\xi_{l,i}$ and will use the value of $\xi_{u,i}$ as the estimate for $\xi_i$. Note that the two $\xi_i$'s converge rapidly to a single value after a small number of training epochs. 


\subsection{Model Bias Offset (MBO)}

Although the constraint reformulation approach described in the previous subsection ensures that the DNN outputs will satisfy the GEV constraints, the random initialization of the network can produce estimates of $\xi$ that violate the regularity conditions described in Section \ref{sec:gev}. Specifically, the ML-estimated distribution may not have the asymptotic GEV distribution when $\xi < -0.5$ while its conditional mean is not well-defined when $\xi > 1$.
Additionally, the estimated location parameter $\mu$ may not fall within the desired range between $y_{\min}$ and $y_{\max}$ when the DNN is randomly initialized. Thus, without proper initialization, the DNN will struggle to converge to a good solution and produce acceptable values of the GEV parameters.

One way to address this challenge is to repeat the random initialization of the DNN until a reasonable set of initial GEV parameters, i.e., $y_{\min} \le \mu \le y_{\max}$ and $-0.5 < \xi < 1$, is found. However, this approach is infeasible given the size of the parameter space of the DNN. A better strategy is to control the initial output of the neural network in order to produce an acceptable set of GEV parameters, $(\mu, \sigma, \xi)$ during initialization. Unfortunately, controlling the initial output of a neural network is difficult given its complex architecture. 

\begin{figure}[t]
\centering
\includegraphics[width=0.9\columnwidth]{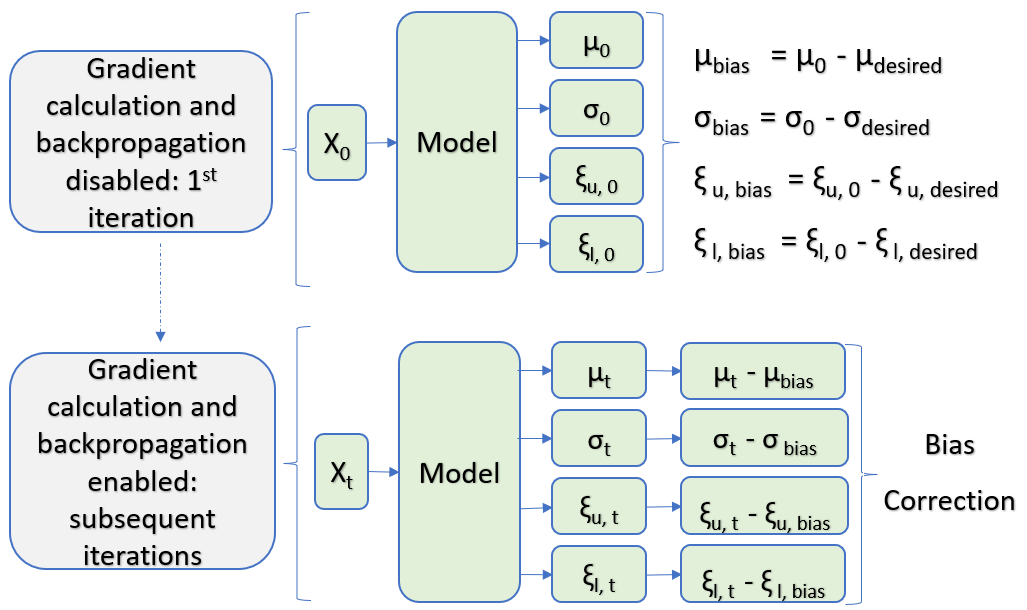} 
\caption{Model Bias Offset (MBO) to ensure the initial estimates of the GEV parameters are reasonable even when the DNN is randomly initialized.}
\label{diof}
\end{figure}

We introduce a simple but effective technique called Model Bias Offset (MBO) to address this challenge. The key insight here is to view the GEV parameters as a biased output due to the random initialization of the DNN and then perform bias correction to alleviate the effect of the initialization. To do this, let $\mu_{\mathrm{desired}}, \sigma_{\mathrm{desired}}$, and  $\xi_{\mathrm{desired}}$ be an acceptable set of initial GEV parameters. The values of these initial parameters must satisfy the regularity conditions  $-0.5 <  \xi_{\mathrm{desired}} < 1$, $\sigma_{\mathrm{desired}}>0$, and $y_{\min} \le \mu_{\mathrm{desired}} \le y_{\max}$. This can be done by randomly choosing a value from the preceding range of acceptable values, or more intelligently, using the GEV parameters estimated from a global GEV distribution fitted to the block maxima $y_i$'s in the training data via the ML approach given in \eqref{loglikelihood}, without considering the input predictors (see the discussion at the beginning of Section \ref{sec:gev_est}). 
We find that the latter strategy 
works well in practice as it can lead to faster convergence especially when the global GEV parameters are close to the final values after training. 


When the DNN is randomly initialized, let $\mu_{\mathrm{0}}, \sigma_{\mathrm{0}}, \xi_{u,\mathrm{0}},$ and $\xi_{l,\mathrm{0}}$ be the initial DNN output for the GEV parameters. These initial outputs may not necessarily fall within their respective range of acceptable values. We consider the difference between the initial DNN output and the desired GEV parameters as a \emph{model bias} due to the random initialization:
\begin{align} 
    &\mu_{\mathrm{bias}} = \mu_{\mathrm{0}} - \mu_{\mathrm{desired}}  
    &\sigma_{\mathrm{bias}} = \sigma_{\mathrm{0}} - \sigma_{\mathrm{desired}} \nonumber\\
    &\xi_{u,\mathrm{bias}} = \xi_{u,\mathrm{0}} - \xi_{u,\mathrm{desired}}
    &\xi_{l,\mathrm{bias}} = \xi_{l,\mathrm{0}} - \xi_{l,\mathrm{desired}}
    \label{diof_fix}
\end{align}

The model bias terms in \eqref{diof_desired} can be computed during the initial forward pass of the algorithm. The gradient calculation and back-propagation are disabled during this step to prevent the DNN from computing the loss and updating its weight with the unacceptable GEV parameters. After the initial iteration, the gradient calculation will be enabled and the bias terms will be subtracted from the DNN estimate of the GEV parameters in all subsequent iterations $t$: 
\begin{align} \label{diof_desired}
    &\mu_{t} \rightarrow \mu_t - \mu_{\mathrm{bias}} 
    &\sigma_{t} \rightarrow \sigma_{t} - \sigma_{\mathrm{bias}}\nonumber \\
    &\xi_{u, t} \rightarrow \xi_{u, t} - \xi_{u, \mathrm{bias}} 
    &\xi_{l, t} \rightarrow \xi_{l, t} - \xi_{l,\mathrm{bias}}
\end{align}
Observe that, when $\mu_t$ is set to $\mu_0$, then the debiased output $\mu_t - \mu_{\mathrm{bias}}$ will be equal to $\mu_{\mathrm{desired}}$.
By debiasing the output of the DNN in this way, we guarantee that the initial GEV parameters are reasonable and satisfy the GEV regularity conditions.


\subsection{Block Maxima Prediction}

Given an input $x_i$, the DNN will estimate the GEV parameters needed to compute the block maxima $\hat{y}_i$ along with its upper and lower quantiles, $\hat{y}_{U,i}$ and $\hat{y}_{L,i}$, respectively. The quantiles are estimated using the formula given in \eqref{quantile}. The GEV parameters are provided as input to a fully connected network (FCN) to generate the block maxima prediction, $\hat{y}_i$. 

\texttt{DeepExtrema} employs a combination of the negative log-likelihood function $(-\ell_{GEV}(\mu, \sigma, \xi))$ of the GEV distribution and a least-square loss function to train the model. This enables the framework to simultaneously learn the GEV parameters and make accurate block maxima predictions. The loss function to be minimized by \texttt{DeepExtrema} is:
\begin{equation} 
    \mathcal{L} = \lambda_1 \: \mathcal{\hat{L}} +   (1 - \lambda_1) \; \sum_{i=1}^n (y_i  -\hat{y_i})^2   \label{loss}
\end{equation}
where $\mathcal{\hat{L}} = - \lambda_2 \; \ell_{GEV}(\mu, \sigma, \xi) + (1-\lambda_2) \sum_{i=1}^n (\xi_{u,i} - \xi_{l,i})^2$ is the regularized GEV loss. The first term in $\mathcal{\hat{L}}$ corresponds to the negative log-likelihood function given in Equation \eqref{loglikelihood} while the second term minimizes the difference between the upper and lower-bound estimates of $\xi$. The loss function $\mathcal{L}$ combines the regularized GEV loss ($\mathcal{\hat{L}}$) with the least-square loss. Here, $\lambda_1 \text{ and } \lambda_2$ are hyperparameters to manage the trade-off between different factors of the loss function.




\section{Experimental Evaluation}
This section presents our experimental results comparing \texttt{DeepExtrema} against the various baseline methods. The code and datasets are available at  \url{https://github.com/galib19/DeepExtrema-IJCAI22-}. 

\subsection{Data}
\subsubsection{Synthetic Data} As the ground truth GEV parameters are often unknown, we have created a synthetic dataset to evaluate the performance of various methods in terms of their ability to correctly infer the parameters of the GEV distribution. The data is generated assuming the GEV parameters are functions of some input predictors $x \in \mathbb{R}^6$. We first generate $x$ by random sampling from a uniform distribution. 
We then assume a non-linear mapping from $x$ to the GEV parameters $\mu$, $\sigma$, and $\xi$, via the following nonlinear equations:  
\begin{align} \label{syn_data}
    & \mu(x) =  w_\mu^T(\exp{(x)}+x) \ \ \ \ 
    & \sigma(x) =  w_\sigma^T(\exp{(x)}+x) \nonumber \\
    & \xi(x) =   w_\xi^T(\exp{(x)}+x)
\end{align}
where $w_\mu$, $w_\sigma$, and $w_\xi$ are generated from a standard normal distribution. Using the generated $\mu$, $\sigma$, and $\xi$ parameters, we then randomly sample $y$ from the GEV distribution governed by the GEV parameters. Here, $y$ denotes the block maxima as it is generated from a GEV distribution. We created 8,192 block maxima values for our synthetic data. 

\subsubsection{Real-world data}
We consider the following 3 datasets for our experiments. 

\paragraph{Hurricane:} This corresponds to tropical cyclone intensity data obtained from the HURDAT2 database~\cite{landsea2013atlantic}. 
There are altogether 3,111 hurricanes spanning the period between 1851 and 2019. For each hurricane, wind speeds (intensities) were reported at every 6-hour interval. 
We consider only hurricanes that have at least 24-time steps at minimum for our experiments. For each hurricane, we have created non-overlapping time windows of length 24 time steps (6 days). We use the first 16 time steps (4 days) in the window as the predictor variables and the block maxima of the last 8 time steps (2 days) as the target variable. 

\paragraph{Solar:} This corresponds to half-hourly energy use (kWh) for 55 families over the course of 284 days from Ausgrid database \cite{Ausgrid_2013}. 
We preprocess the data by creating non-overlapping time windows of length 192 time steps (4 days). We use the first 144 time steps (3 days) in the window as the predictor variables and the block maxima of the last 48 time steps (1 day) as the target variable.  

\paragraph{Weather:} We have used a weather dataset from the Kaggle competition~\cite{muthukumar.j_2017}. The data is based on hourly temperature data for a city over a ten-year period. We use the first 16 time steps (16 hours) in the window as the predictor variables and the block maxima of the last 8 time steps (8 hours) as the target variable.

\subsection{Experimental Setup}

For evaluation purposes, we split the data into separate training, validation, testing with a ratio of 7:2:1. The data is standardized to have zero mean and unit variance. We compare \texttt{DeepExtrema} against the following baseline methods: (1) Persistence, which uses the block maxima value from the previous time step as its predicted value, (2) fully-connected network (FCN), (3) LSTM, (4) Transformer, (5) DeepPIPE \cite{wang2020deeppipe}, and (6) EVL \cite{ding2019modeling}. We will use the following metrics to evaluate the performance of the methods: (1) Root mean squared error (RMSE) and correlation between the predicted and ground truth block maxima and (2) Negative log-likelihood (for synthetic data). 
Finally, hyperparameter tuning is performed by assessing the model performance on the validation set. The hyperparameters of the baseline and the proposed methods are selected using Ray Tune, a tuning framework with ASHA (asynchronous successive halving algorithm) scheduler for early stopping. 

\subsection{Experimental Results}

\begin{table}[t!]
\begin{tabular}{|c|c|c|}
\hline
\multicolumn{3}{|c|}{Negative Log-likelihood}           \\ \hline
DeepExtrema & Ground Truth &  Global GEV Estimate \\ \hline
\textbf{4410}  & 4451    & 4745             \\ \hline
\end{tabular}
\caption{Negative log-likelihood of \texttt{DeepExtrema} with respect to ground truth and global parameter estimation using synthetic data}
\label{log-likelihood_table}
\end{table}

\subsubsection{Results on Synthetic Data} 

In this experiment, we have compared the performance of \texttt{DeepExtrema} against using a single (global) GEV parameter to fit the data. Based on the results shown in Table \ref{log-likelihood_table}, \texttt{DeepExtrema} achieves a significantly lower negative log-likelihood of 4410 compared to the negative log-likelihood for global GEV estimate, which is 4745. This result supports the assumption that each block maxima comes from different GEV distributions rather than a single (global) GEV distribution. The results also suggest that the negative log likelihood estimated by \texttt{DeepExtrema} 
is lower than that for the ground truth. 


\begin{table}[h]
\resizebox{\columnwidth}{!}{%
\begin{tabular}{|c|cc|cc|cc|}
\hline
\multirow{2}{*}{Methods} & \multicolumn{2}{c|}{Hurricanes}                     & \multicolumn{2}{c|}{Ausgrid}                       & \multicolumn{2}{c|}{Weather}                       \\ \cline{2-7} 
                         & \multicolumn{1}{c|}{RMSE}           & Correlation   & \multicolumn{1}{c|}{RMSE}          & Correlation   & \multicolumn{1}{c|}{RMSE}          & Correlation   \\ \hline
Persistence              & \multicolumn{1}{c|}{28.6}           & 0.6           & \multicolumn{1}{c|}{0.84}          & 0.65          & \multicolumn{1}{c|}{4.16}          & 0.96          \\ \hline
FCN                      & \multicolumn{1}{c|}{14.14}          & 0.87          & \multicolumn{1}{c|}{0.69}          & 0.65          & \multicolumn{1}{c|}{2.5}           & 0.97          \\ \hline
LSTM                     & \multicolumn{1}{c|}{13.31}          & 0.88          & \multicolumn{1}{c|}{0.65}          & 0.64          & \multicolumn{1}{c|}{2.53}          & 0.97          \\ \hline
Transformer              & \multicolumn{1}{c|}{13.89}          & 0.88          & \multicolumn{1}{c|}{0.68}          & 0.62          & \multicolumn{1}{c|}{2.43}          & \textbf{0.98} \\ \hline
DeepPIPE                 & \multicolumn{1}{c|}{13.67}          & 0.87          & \multicolumn{1}{c|}{0.71}          & 0.59          & \multicolumn{1}{c|}{2.59}          & 0.94          \\ \hline
EVL         & \multicolumn{1}{c|}{15.72}          & 0.83          & \multicolumn{1}{c|}{0.75}          & 0.54          & \multicolumn{1}{c|}{2.71}          & 0.90          \\ \hline
DeepExtrema              & \multicolumn{1}{c|}{\textbf{12.81}} & \textbf{0.90} & \multicolumn{1}{c|}{\textbf{0.63}} & \textbf{0.67} & \multicolumn{1}{c|}{\textbf{2.27}} & 0.97          \\ \hline
\end{tabular}
}
\caption{Performance comparison on real world data.} 
\label{nhc_results}
\end{table}

\subsubsection{Results on Real-world Data}

Evaluation on real-world data shows that \texttt{DeepExtrema} outperforms other baseline methods used for comparison for all data sets (see Table \ref{nhc_results}). For RMSE, \texttt{DeepExtrema} generates lower RMSE compared to all the baselines on all 3 datasets, whereas for correlation, \texttt{DeepExtrema} outperforms the baselines on 2 of the 3 datasets. 


\begin{figure}[t!]
\centering
\includegraphics[width=0.75\columnwidth]{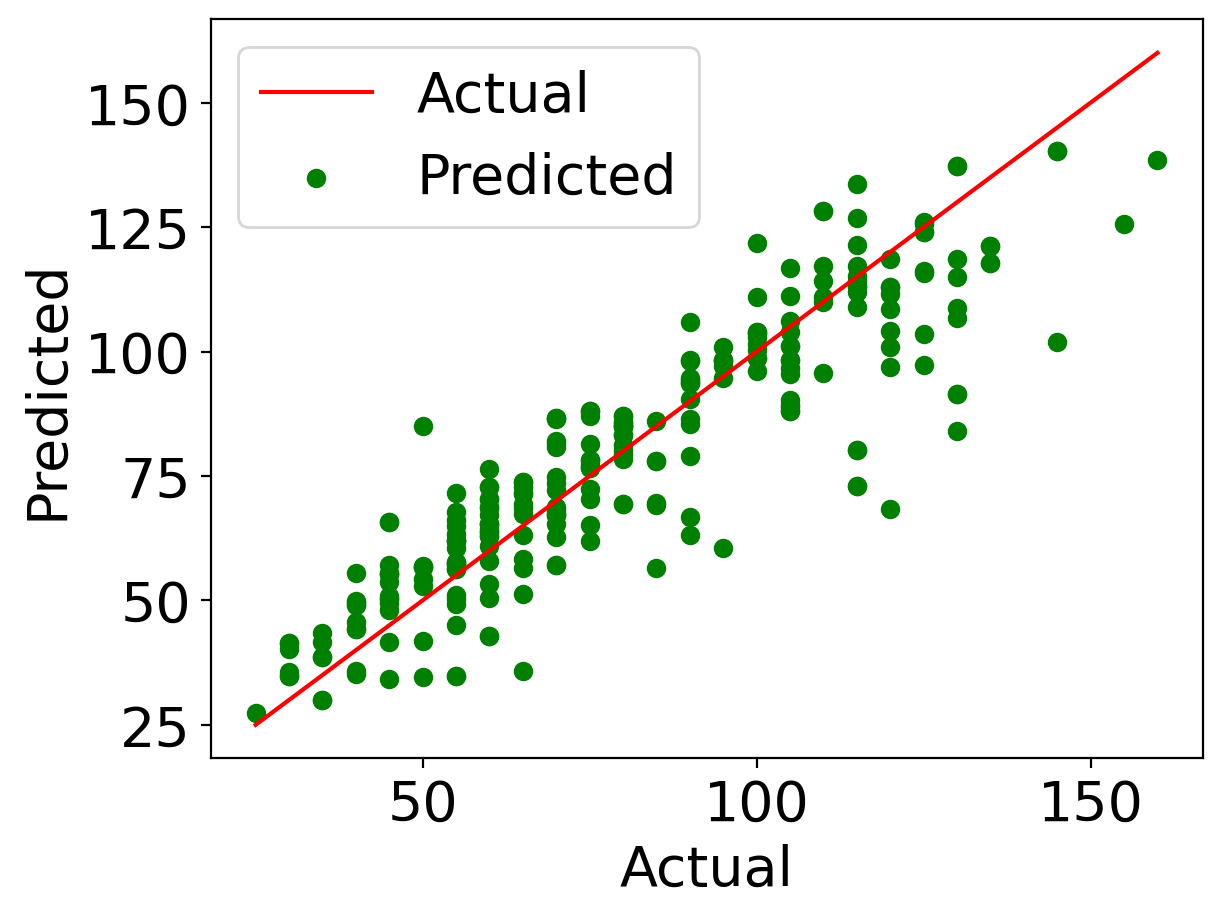} 
\caption{Comparison between actual and predicted block maxima of hurricane intensities for \texttt{DeepExtrema}.}
\label{nhc_line}
\end{figure}

\begin{figure}[t!]
\centering
\includegraphics[width=0.75\columnwidth]{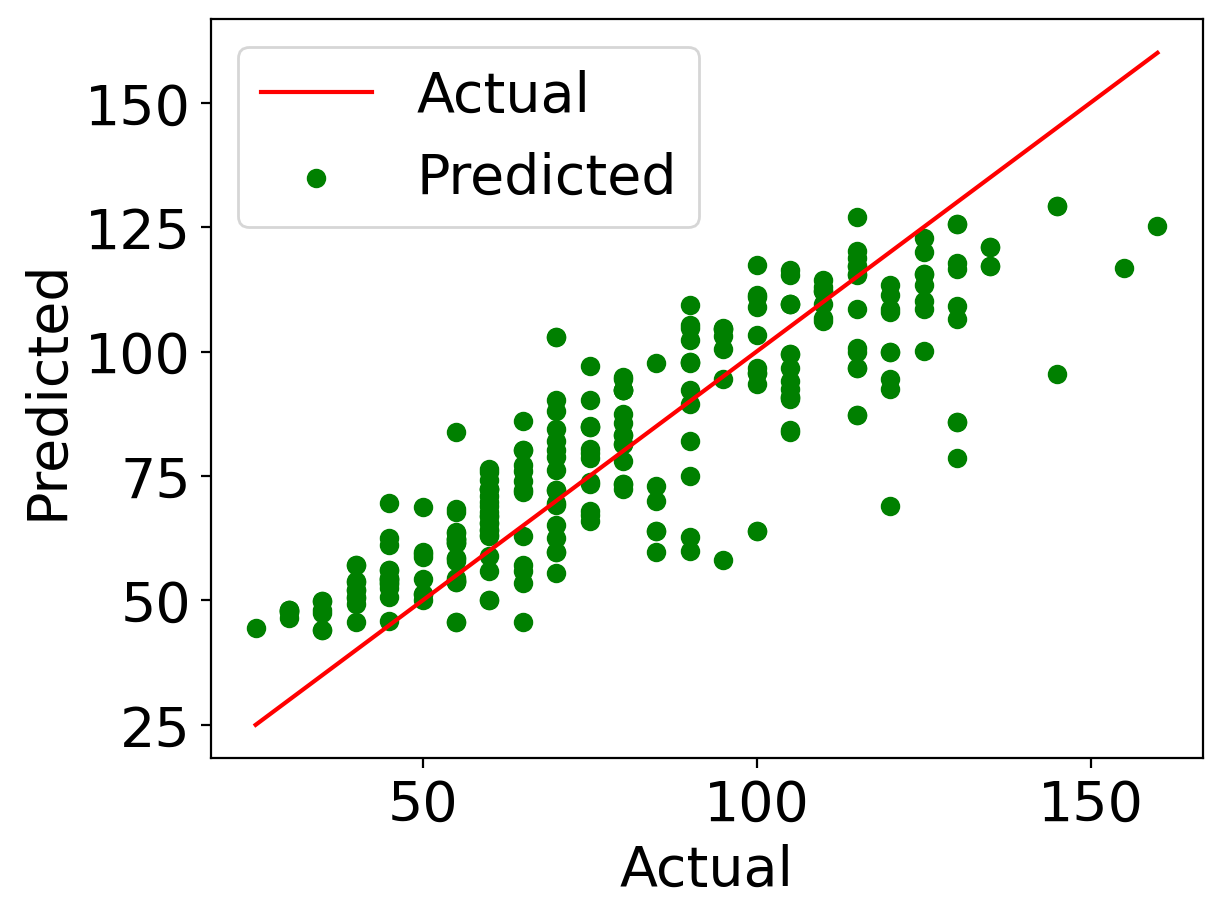} 
\caption{Comparison between actual and predicted block maxima of hurricane intensities for \texttt{EVL} \protect\cite{ding2019modeling}.}
\label{evl_result}
\end{figure}

\begin{figure}[t!]
\centering
\includegraphics[width=0.75\columnwidth]{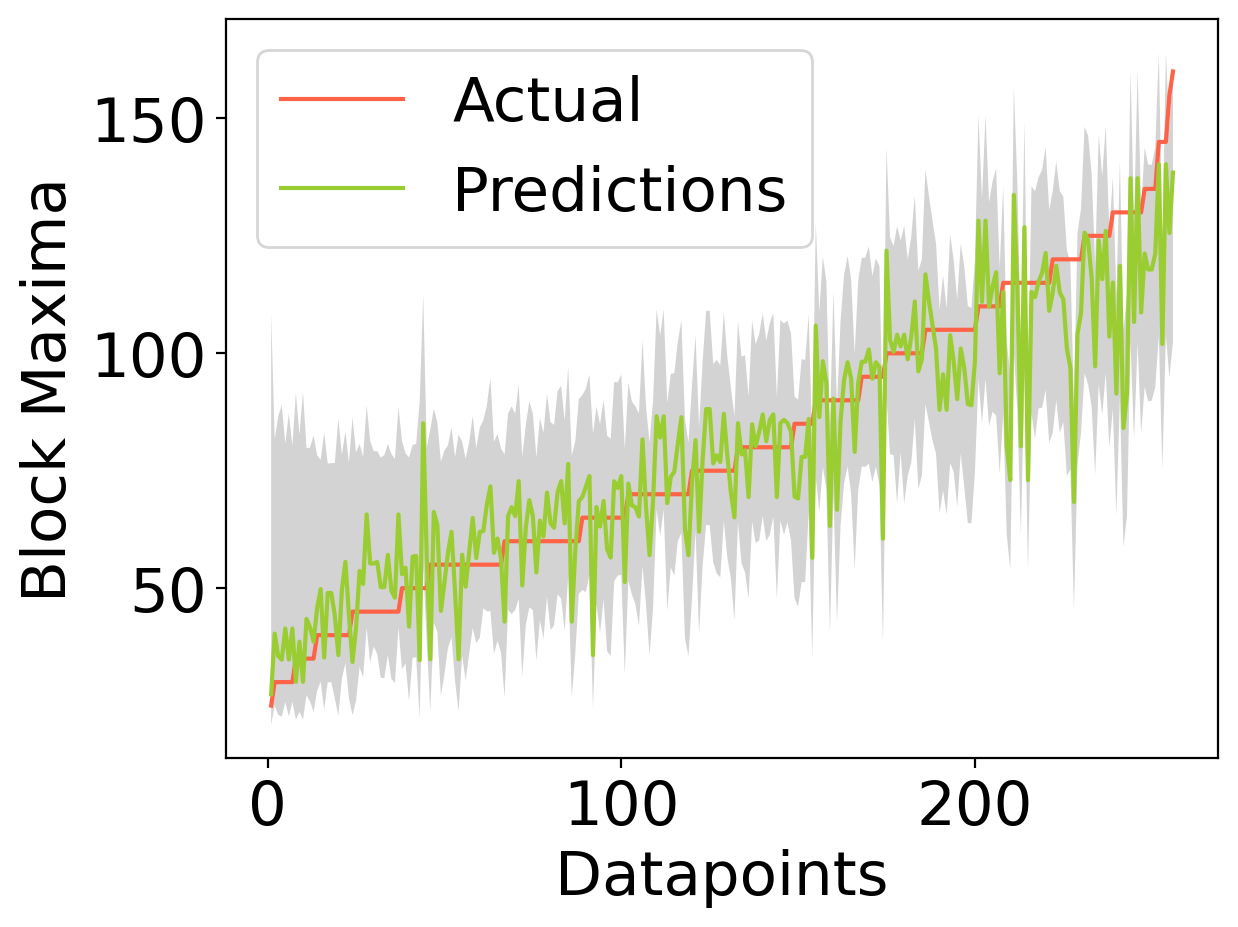} 
\caption{90\% confidence interval of the hurricane intensity predictions for \texttt{DeepExtrema}, sorted in increasing block maxima values. Ground truth values are shown in red.}
\label{nhc_quantile}
\end{figure}

To demonstrate how well the model predicts the extreme values, Figure \ref{nhc_line} shows a scatter plot of the actual versus predicted values generated by \texttt{DeepExtrema} on the test set of the hurricane intensity data. The results suggest that \texttt{DeepExtrema} can accurately predict the hurricane intensities for a wide range of values, especially those below 140 knots. \texttt{DeepExtrema} also does a better job at predicting the high intensity hurricanes compared to \texttt{EVL} \cite{ding2019modeling}, as illustrated in  Figure \ref{evl_result}.
Figure \ref{nhc_quantile} shows the 90\% confidence interval of the predictions.  
Apart from the point and quantile estimations, \texttt{DeepExtrema} can also estimate the GEV parameter values for each hurricane. 

\subsubsection{Ablation Studies}

The hyperparameter $\lambda_1$ in the objective function of \texttt{DeepExtrema} denotes the trade-off between regularized GEV loss and RMSE loss. Experimental results show that the RMSE of block maxima prediction decreases when $\lambda_1$ increases (see Table \ref{ablation}). This validates the importance of incorporating GEV theory to improve the accuracy of block maxima estimation instead of using the mean squared loss alone.   

\begin{table}[H]
\begin{center}
\begin{tabular}{|c|ccc|}
\hline
\multirow{2}{*}{\begin{tabular}[c]{@{}c@{}}Hyperparameter\end{tabular}} & \multicolumn{3}{c|}{\begin{tabular}[c]{@{}c@{}}RMSE of Block maxima\end{tabular}}     \\ \cline{2-4} 
                                                                                & \multicolumn{1}{c|}{Hurricanes}     & \multicolumn{1}{c|}{Ausgrid}       & Weather       \\ \hline
$\lambda_1  = 0.0$                                                                             & \multicolumn{1}{c|}{13.28}          & \multicolumn{1}{c|}{0.68}          & 2.52          \\ \hline
$\lambda_1   = 0.5$                                                                            & \multicolumn{1}{c|}{13.03}          & \multicolumn{1}{c|}{\textbf{0.63}} & 2.41          \\ \hline
$\lambda_1   = 0.9$                                                                            & \multicolumn{1}{c|}{\textbf{12.81}} & \multicolumn{1}{c|}{0.64}          & \textbf{2.27} \\ \hline
\end{tabular}
\caption{Effect of $\lambda_1$ on RMSE of block maxima prediction.}
\label{ablation}
\end{center}
\end{table}


\section{Conclusion}

This paper presents a novel deep learning framework called \texttt{DeepExtrema} that combines extreme value theory with deep learning to address the challenges of predicting extremes in time series. We offer a reformulation and re-parameterization technique for satisfying constraints as well as a model bias offset technique for proper model initialization. We evaluated our framework on synthetic and real-world data and showed its effectiveness. For future work, we plan to extend the formulation to enable more complex deep learning architectures such as those using an attention mechanism. 
In addition, the framework will be extended to model extremes in spatio-temporal data.

\section{Acknowledgments}

This research is supported by the U.S. National Science Foundation under grant IIS-2006633. Any use of trade, firm, or product names is for descriptive purposes only and does not imply endorsement by the U.S. Government.

\bibliographystyle{named}
\bibliography{ijcai22}

\begin{thebibliography}{}

\bibitem[\protect\citeauthoryear{Aliabadi \bgroup \em et al.\egroup
  }{2020}]{aliabadi2020attention}
Majid~Moradi Aliabadi, Hajar Emami, Ming Dong, and Yinlun Huang.
\newblock Attention-based recurrent neural network for multistep-ahead
  prediction of process performance.
\newblock {\em Computers \& Chemical Engineering}, 140:106931, 2020.

\bibitem[\protect\citeauthoryear{Aus}{2013}]{Ausgrid_2013}
Solar home electricity data, Dec 2013.

\bibitem[\protect\citeauthoryear{Bai \bgroup \em et al.\egroup
  }{2018}]{bai-empirical-2018}
Shaojie Bai, J.~Zico Kolter, and Vladlen Koltun.
\newblock An {Empirical} {Evaluation} of {Generic} {Convolutional} and
  {Recurrent} {Networks} for {Sequence} {Modeling}.
\newblock {\em arXiv:1803.01271 [cs]}, April 2018.
\newblock arXiv: 1803.01271.

\bibitem[\protect\citeauthoryear{Bishop}{2006}]{bishop2006pattern}
Christopher~M Bishop.
\newblock {\em Pattern recognition}.
\newblock Springer, New York, 2006.

\bibitem[\protect\citeauthoryear{Coles \bgroup \em et al.\egroup
  }{2001}]{coles2001introduction}
Stuart Coles, Joanna Bawa, Lesley Trenner, and Pat Dorazio.
\newblock {\em An introduction to statistical modeling of extreme values},
  volume 208.
\newblock Springer, 2001.

\bibitem[\protect\citeauthoryear{Ding \bgroup \em et al.\egroup
  }{2019}]{ding2019modeling}
Daizong Ding, Mi~Zhang, Xudong Pan, Min Yang, and Xiangnan He.
\newblock Modeling extreme events in time series prediction.
\newblock In {\em Proceedings of the 25th ACM SIGKDD International Conference
  on Knowledge Discovery \& Data Mining}, pages 1114--1122, 2019.

\bibitem[\protect\citeauthoryear{Kharin and
  Zwiers}{2005}]{kharin2005estimating}
Viatcheslav~V Kharin and Francis~W Zwiers.
\newblock Estimating extremes in transient climate change simulations.
\newblock {\em Journal of Climate}, 18(8):1156--1173, 2005.

\bibitem[\protect\citeauthoryear{Landsea and
  Franklin}{2013}]{landsea2013atlantic}
Christopher~W Landsea and James~L Franklin.
\newblock Atlantic hurricane database uncertainty and presentation of a new
  database format.
\newblock {\em Monthly Weather Review}, 141(10):3576--3592, 2013.

\bibitem[\protect\citeauthoryear{Laptev \bgroup \em et al.\egroup
  }{2017}]{laptev2017time}
Nikolay Laptev, Jason Yosinski, Li~Erran Li, and Slawek Smyl.
\newblock Time-series extreme event forecasting with neural networks at uber.
\newblock In {\em International conference on machine learning}, volume~34,
  pages 1--5, 2017.

\bibitem[\protect\citeauthoryear{Masum \bgroup \em et al.\egroup
  }{2018}]{masum2018multi}
Shamsul Masum, Ying Liu, and John Chiverton.
\newblock Multi-step time series forecasting of electric load using machine
  learning models.
\newblock In {\em International conference on artificial intelligence and soft
  computing}, pages 148--159. Springer, 2018.

\bibitem[\protect\citeauthoryear{Muthukumar}{2017}]{muthukumar.j_2017}
J~Muthukumar.
\newblock Weather dataset, Dec 2017.

\bibitem[\protect\citeauthoryear{Peng \bgroup \em et al.\egroup
  }{2018}]{peng2018multi}
Chenglei Peng, Yang Li, Yao Yu, Yu~Zhou, and Sidan Du.
\newblock Multi-step-ahead host load prediction with gru based encoder-decoder
  in cloud computing.
\newblock In {\em 2018 10th International Conference on Knowledge and Smart
  Technology (KST)}, pages 186--191. IEEE, 2018.

\bibitem[\protect\citeauthoryear{Polson and Sokolov}{2020}]{polson2020deep}
Michael Polson and Vadim Sokolov.
\newblock Deep learning for energy markets.
\newblock {\em Applied Stochastic Models in Business and Industry},
  36(1):195--209, 2020.

\bibitem[\protect\citeauthoryear{Sagheer and Kotb}{2019}]{sagheer2019time}
Alaa Sagheer and Mostafa Kotb.
\newblock Time series forecasting of petroleum production using deep lstm
  recurrent networks.
\newblock {\em Neurocomputing}, 323:203--213, 2019.

\bibitem[\protect\citeauthoryear{Smith}{1985}]{smith1985}
Richard~L. Smith.
\newblock Maximum likelihood estimation in a class of nonregular cases.
\newblock {\em Biometrika}, 72(1):67–90, 1985.

\bibitem[\protect\citeauthoryear{{US GAO}}{2020}]{usgao}
{US GAO}.
\newblock Natural disasters: Economic effects of hurricanes katrina, sandy,
  harvey, and irma.
\newblock \url{https://www.gao.gov/products/gao-20-633r}, 2020.
\newblock Accessed: 2021-04-09.

\bibitem[\protect\citeauthoryear{Wang and Tan}{2021}]{wang2021johan}
Ding Wang and Pang-Ning Tan.
\newblock Johan: A joint online hurricane trajectory and intensity forecasting
  framework.
\newblock In {\em Proceedings of the 27th ACM SIGKDD Conference on Knowledge
  Discovery \& Data Mining}, pages 1677--1685, 2021.

\bibitem[\protect\citeauthoryear{Wang \bgroup \em et al.\egroup
  }{2020}]{wang2020deeppipe}
Bin Wang, Tianrui Li, Zheng Yan, Guangquan Zhang, and Jie Lu.
\newblock Deeppipe: A distribution-free uncertainty quantification approach for
  time series forecasting.
\newblock {\em Neurocomputing}, 397:11--19, 2020.

\bibitem[\protect\citeauthoryear{Wilson \bgroup \em et al.\egroup
  }{2022}]{wilson2022}
Tyler Wilson, Pang-Ning Tan, and Lifeng Luo.
\newblock Deepgpd: A deep learning approach for modeling geospatio-temporal
  extreme events.
\newblock In {\em {Proceedings of the 36th AAAI Conference on Artificial
  Intelligence}}, 2022.

\bibitem[\protect\citeauthoryear{Yang \bgroup \em et al.\egroup
  }{2015}]{yang2015deep}
Jianbo Yang, Minh~Nhut Nguyen, Phyo~Phyo San, Xiao~Li Li, and Shonali
  Krishnaswamy.
\newblock Deep convolutional neural networks on multichannel time series for
  human activity recognition.
\newblock In {\em Twenty-fourth international joint conference on artificial
  intelligence}, 2015.

\bibitem[\protect\citeauthoryear{Zhang \bgroup \em et al.\egroup
  }{2019}]{zhang2019lstm}
Xuan Zhang, Xun Liang, Aakas Zhiyuli, Shusen Zhang, Rui Xu, and Bo~Wu.
\newblock At-lstm: An attention-based lstm model for financial time series
  prediction.
\newblock In {\em IOP Conference Series: Materials Science and Engineering},
  volume 569, page 052037. IOP Publishing, 2019.

\bibitem[\protect\citeauthoryear{Zhao \bgroup \em et al.\egroup
  }{2017}]{zhao2017convolutional}
Bendong Zhao, Huanzhang Lu, Shangfeng Chen, Junliang Liu, and Dongya Wu.
\newblock Convolutional neural networks for time series classification.
\newblock {\em Journal of Systems Engineering and Electronics}, 28(1):162--169,
  2017.

\end{thebibliography}

\end{document}